\documentclass{colt2016} % Anonymized submission
\usepackage{enumitem}

\title[Nearly optimal pseudo-regret for stochastic and adversarial bandits]{An algorithm with nearly optimal pseudo-regret \\for both stochastic and adversarial bandits}
\usepackage{times}
 \coltauthor{\Name{Peter Auer} \thanks{Accepted for presentation at the Conference on Learning Theory (COLT) 2016.} \Email{auer@unileoben.ac.at}\\
 	\addr Chair for Information Technology \\ 
 	Montanuniversitaet Leoben\\ Franz-Josef-Strasse 18, A-8700 Leoben, Austria 
 	\AND 
 	\Name{Chao-Kai Chiang} \Email{chaokai@gmail.com}\\
 	\addr University of California, Los Angeles \\ 
 	Computer Science Department \\ 4732 Boelter Hall \\ Los Angeles, CA 90095
 }
\newcommand{\cB}{\mathcal{B}}
\newcommand{\expp}{\textsc{Exp3.P}}

\newcommand{\clower}{C}

\newcommand{\Cgap}{C_\mathrm{gap}}
\newcommand{\Cib}{C_\mathrm{1b}}
\newcommand{\Caa}{C_\mathrm{4a}}

\newcommand{\CVV}{2\Caa}

\newcommand{\Ci}{C_\mathrm{init}}
\newcommand{\Cpp}{C_\mathrm{p}}

\newcommand{\Cwid}{C_\mathrm{w}}

\newcommand{\ci}{C_\mathrm{E}}

\newcommand{\ns}{{n_S}}
\newcommand{\nbi}{{n_{B,i}}}

\newcommand{\hD}{\hat{D}}

\newcommand{\hmu}{\hat{\mu}}
\newcommand{\bmu}{\bar{\mu}}

\newcommand{\tO}[1]{\tilde{O}\left(#1\right)}
\newcommand{\bigO}[1]{O\left(#1\right)}
\newcommand{\EE}[1]{\bE\left[#1\right]}
\newcommand{\VV}[1]{\bV\left[#1\right]}

\newcommand{\Esto}[1]{\bE_\sto\left[#1\right]}
\newcommand{\Psto}[1]{\bP_\sto\left\{#1\right\}}
\newcommand{\Eadv}[1]{\bE_\adv\left[#1\right]}
\newcommand{\Padv}[1]{\bP_\adv\left\{#1\right\}}
\newcommand{\EEik}[1]{\bE_{ik}\left[#1\right]}

\newcommand{\PPik}[1]{\bP_{ik}\left\{#1\right\}}
\newcommand{\II}[1]{\mathbb{I}\left[#1\right]}

\newcommand{\qsto}{{q_\sto}}
\newcommand{\qad}[1]{{q_\adv^{#1}}}
\newcommand{\qadv}{{\qad{}}}

\usepackage{algorithm}
\usepackage{algorithmic}

\newcommand{\pr}[1]{\mathbb{P}\left\{#1\right\}} %--- probability
\newcommand{\expt}[1]{\mathbb{E}\mpa{#1}} %--- expected value
\newcommand{\bE}{\mathbb{E}} %--- expected
\newcommand{\bV}{\mathbb{V}} %--- expected
\newcommand{\bP}{\mathbb{P}} %--- expected
\newcommand{\Id}[1]{\mathbb{I}[#1]} %--- indicator random variable
\newcommand{\cA}{\mathcal{A}} %--- algorithm or player
\newcommand{\T}{n} %--- total number of rounds
\newcommand{\width}{{\mathrm{width}}} % width
\newcommand{\bwidth}{\overline{\mathrm{width}}} % width

\newcommand{\bucb}{\overline{\mathrm{ucb}}}
\newcommand{\lcb}{\mathrm{lcb}}
\newcommand{\blcb}{\overline{\mathrm{lcb}}}

\newcommand{\emu}{\tilde{\mu}} %--- \mu computed in exploration stage
\newcommand{\egp}{\tilde{\Delta}} %--- estimated gap
\newcommand{\Gitt}[3]{G_{#1}(#2, #3)}

\newcommand{\GAitt}[3]{G_{\cA,#1}(#2, #3)}

\newcommand{\DAitt}[3]{D_{\cA,#1}(#2, #3)} %--- deviation of samples from \emu_i

\newcommand{\sapo}{\textrm{SAPO}} % \algo{SAO_2}
\newcommand{\spa}[1]{\left(#1\right)} %--- small parentheses
\newcommand{\mpa}[1]{\left[#1\right]} %--- M parentheses
\newcommand{\pseudoR}{{\overline{R}}}
\newcommand{\adv}{{\mathrm{adv}}}
\newcommand{\sto}{{\mathrm{sto}}}
\newcommand{\ada}{{\mathrm{ada}}}
\newcommand{\obl}{{\mathrm{obl}}}

\newcommand{\CASE}[1]{{\em Case #1}}

\newcommand{\tC}{{t_C}}

\newcommand{\hist}{{\cal H}}
\newcommand{\ND}{{\mathit{ND}}}
\newcommand{\numPH}{M}
\jmlrvolume{}
\jmlryear{}

\begin{document}

\maketitle

\begin{abstract}
	We present an algorithm that achieves almost optimal pseudo-regret bounds against adversarial and stochastic bandits. Against adversarial bandits the pseudo-regret is $\bigO{\sqrt{K n \log n}}$ and against stochastic bandits the regret is $\bigO{\sum_i (\log n)/\Delta_i}$. We also show that no algorithm with $\bigO{\log n}$ pseudo-regret against stochastic bandits can achieve $\tO{\sqrt{n}}$ expected regret against adaptive adversarial bandits. This complements previous results of \cite{BS12} that show $\tO{\sqrt{n}}$ expected adversarial regret with $\bigO{(\log n)^2}$ stochastic pseudo-regret. 
\end{abstract}

\iffalse
\fi

\section{Introduction}\label{sec_intro}

We consider the multi-armed bandit problem, which is the most basic example of a sequential decision problem with an exploration-exploitation trade-off. In each time step $t=1,2,\ldots,n$, the player has to play an arm $I_t \in \{1,\ldots,K\}$ from this fixed finite set and receives  reward~$x_{I_t}(t) \in[0,1]$ depending on its choice\footnote{%TODO
	We assume that the player knows the total number of time steps~$n$.}. 
The player observes only the reward of the chosen arm, but not the rewards of the other arms $x_i(t)$, $i \neq I_t$. The player's goal is to maximize its total reward $\sum_{t=1}^n x_{I_t}(t)$, and this total reward is compared to the best total reward of a single arm, $\sum_{t=1}^n x_i(t)$. To identify the best arm the player needs to explore all arms by playing them, but it also needs to limit this exploration to often play the best arm. The optimal amount of exploration constitutes the exploration-exploitation trade-off.

Different assumptions on how the rewards $x_i(t)$ are generated have led to different approaches and algorithms for the multi-armed bandit problem. In the original formulation \citep{R52} it is assumed that the rewards are generated independently at random, governed by fixed but unknown probability distributions with means~$\mu_i$ for each arm~$i=1,\ldots,K$. This type of bandit problem is called {\em stochastic}. The other type of bandit problem that we consider in this paper is called non-stochastic or {\em adversarial}~\citep{EXP3}. Here the rewards may be selected arbitrarily by an adversary and the player should still perform well for any selection of rewards. An extensive overview of multi-armed bandit problems is given in~\citep{BC12}.      
 
A central notion for the analysis of stochastic and adversarial bandit problems is the regret~$R(n)$, the difference between the total reward of the best arm and the total reward of the player:
\[ R(n) = \max_{1\leq i \leq K} \sum_{t=1}^n x_i(t) - \sum_{t=1}^nx_{I_t}(t) . \] 
Since the player does not know the best arm beforehand and needs to do exploration, we expect that the total reward of the player is less than the total reward of the best arm. Thus the regret is a measure for the cost of not knowing the best arm. In the analysis of bandit problems we are interested in high probability bounds on the regret or in bounds on the expected regret. Often it is more convenient, though, to analyze the pseudo-regret 
\[ \pseudoR(n) = \max_{1\leq i \leq K} \EE{\sum_{t=1}^n x_i(t) - \sum_{t=1}^nx_{I_t}(t)}  \] 
instead of the expected regret
\[ \EE{R(n)} = \EE{\max_{1\leq i \leq K} \sum_{t=1}^n x_i(t) - \sum_{t=1}^nx_{I_t}(t)} . \]

While the notion of pseudo-regret is weaker than the expected regret with $\pseudoR(n) \leq \EE{R(n)}$, bounds on the pseudo-regret imply bounds on the expected regret for adversarial bandit problems with {\em oblivious} rewards $x_i(t)$ selected independently from the player's choices. The pseudo-regret also allows for refined bounds in stochastic bandit problems.

\subsection{Previous results}

For adversarial bandit problems, algorithms with high probability bounds on the regret are known~\cite[Theorem 3.3]{BC12}: with probability $1-\delta$,
\[ R_\adv(n) = \bigO{\sqrt{n \log(1/\delta)}}. \]
For stochastic bandit problems, several algorithms achieve logarithmic bounds on the pseudo-regret, e.g.~\cite{UCB1}: 
\[ \pseudoR_\sto(n) = \bigO{\log n}. \]
Both of these bounds are known to be best possible. 

While the result for adversarial bandits is a worst-case --- and thus possibly pessimistic --- bound that holds for any sequence of rewards, the strong assumptions for stochastic bandits may sometimes be unjustified. 
Therefore an algorithm that can adapt to the actual difficulty of the problem is of great interest. 
The first such result was obtained by \cite{BS12}, who developed the SAO algorithm that with probability $1-\delta$ achieves
\[ R_\adv(n) \leq \bigO{(\log n)\sqrt{n\log (n/\delta)}} \]
regret for adversarial bandits and 
\[ \pseudoR_\sto(n) = \bigO{(\log n)^2} \]
pseudo-regret for stochastic bandits. 

It has remained as an open question if a stochastic pseudo-regret of order $\bigO{(\log n)^2}$ is necessary or if the optimal $\bigO{\log n}$ pseudo-regret can be achieved while maintaining an adversarial regret of order~$\sqrt{n}$.

\subsection{Summary of new results}

We give a twofold answer to this open question. We show that stochastic pseudo-regret of order $\bigO{(\log n)^2}$ is necessary for a player to achieve high probability adversarial regret of order~$\sqrt{n}$ against an oblivious adversary, and to even achieve expected regret of order~$\sqrt{n}$ against an adaptive adversary. But we also show that a player can achieve $\bigO{\log n}$ stochastic pseudo-regret and $\tO{\sqrt{n}}$ adversarial {\em pseudo-regret} at the same time. This gives, together with the results of~\citep{BS12}, a quite complete characterization of algorithms that perform well both for stochastic and adversarial bandit problems.

More precisely, for any player with stochastic pseudo-regret bound of order $\bigO{(\log n)^\beta}$, $\beta < 2$, and any $\epsilon >0$, $\alpha < 1$, there is an adversarial bandit problem for which the player suffers $\Omega(n^\alpha)$ regret with probability $\Omega(n^{-\epsilon})$. Furthermore, there is an adaptive adversary against which the player suffers $\Omega(n^\alpha)$ expected regret. Secondly, we construct an algorithm with 
\[ \pseudoR_\sto(n) = \bigO{\log n}  \]
and
\[ \pseudoR_\adv(n) = \bigO{\sqrt{n \log n}} . \]

At first glance these two results may appear contradictory for $\alpha - \epsilon > 1/2$, as the lower bound seems to suggest a pseudo-regret of $\Omega(n^{\alpha-\epsilon})$. This is not the case, though, since the regret may also be negative. Indeed, consider an adversarial multi-armed bandit that initially gives higher rewards for one arm, and from some time step on gives higher rewards for a second arm. A player that detects this change and initially plays the first arm and later the second arm, may outperform both arms and achieve negative regret. But if the player misses the change and keeps playing the first arm, it may suffer  large regret against the second arm.  

In our analysis we use both mechanisms. For the lower bound on the pseudo-regret we show that a player with little exploration (which is necessary for small stochastic pseudo-regret) will miss such a change with significant probability and then will suffer large regret. For the upper bound we explicitly compensate possible large regret that occurs with small probability by negative regret that occurs with sufficiently large probability. For the lower bound on the expected regret we construct an adaptive adversary that prevents such negative regret.    
Consequently, our results exhibit one of the rare cases where there is a significant gap between the achievable pseudo-regret and the achievable expected regret. 

The explicit consideration of negative regret is one of the technical contributions of this work. Another, maybe even more significant contribution, is a weak testing scheme for non-stochastic arms. 
This weak testing scheme is necessary since $\bigO{\log n}$ stochastic pseudo-regret allows only for very little exploration. 
Each individual weak test has a constant false positive rate (predicting a non-stochastic arm although the arm is stochastic) and a constant false negative rate (missing a non-stochastic arm).
To avoid classifying a stochastic arm as non-stochastic, an arm is classified as non-stochastic only after $\bigO{\log n}$ positive tests. This reduces the false positive rate of a decision to acceptable $\bigO{1/n}$.  
Conversely, this delayed detection needs to be accounted for in the regret analysis when the arms are indeed non-stochastic.

\section{Definitions and statement of results}

In a multi-armed bandit problem with arms $i=1,\ldots,K$ the interaction of a player with its environment is governed by the following protocol:
\begin{quote}
For time steps $t=1,\ldots,n$:
\begin{enumerate}
	\item The player chooses an arm $I_t \in \{1,\ldots,K\}$, possibly using randomization. 
	\item The player receives and observes the reward $x_{I_t}(t)$.\\
	      It does not observe the reward from any other arm $i \neq I_t$.
\end{enumerate}
\end{quote}
The player's choice~$I_t$ may depend only on information available at this time, namely $I_1,\ldots,I_{t-1}$ and $x_{I_1}(1),\ldots,x_{I_{t-1}}(t-1)$.
If the bandit problem is stochastic, then the rewards $x_i(t)$ are generated independently at random. If the bandit problem is adversarial, then the rewards are generated arbitrarily by an adversary. We assume that all rewards $x_i(t) \in [0,1]$ and that the number of time steps~$n$ is known to the player.
%TODO

\subsection{Stochastic multi-armed bandit problems}

In a stochastic multi-armed bandit problem the rewards for each arm~$i$ are generated by a fixed but unknown probability distribution $\nu_i$ on~$[0,1]$. All rewards $x_i(t)$, $1 \leq i \leq K$, $1 \leq t \leq n$, are generated independently at random with $x_i(t) \sim \nu_i$. 

Important quantities are the average rewards of the arms, $\mu_i = \EE{x_i(t)}$, the average reward of the best arm $\mu^*=\max_i \mu_i$, and the resulting gaps $\Delta_i=\mu^*-\mu_i$. 

The goal of the player is to achieve low pseudo-regret which for a stochastic bandit problem can be written as
\begin{align*}
    \pseudoR_\sto(n)  
    = \max_{1\leq i \leq K} \EE{\sum_{t=1}^n x_i(t) - \sum_{t=1}^n x_{I_t}(t)} 
    = \sum_{i=1}^K \Delta_i \EE{T_i(n)} ,
\end{align*}
where $T_i(n) = \#\{1 \leq t \leq n: I_t = i\}$ is the number of plays of arm~$i$. It can be shown~\citep{UCB1} that --- among others --- upper confidence bound algorithms achieve 
\[ \EE{T_i(n)} = \bigO{\frac{\log n}{\Delta_i^2}} \]
for any arm $i$ with $\Delta_i > 0$ such that  
\[ \pseudoR_\sto(n) = \bigO{\sum_{i: \Delta_i > 0}\frac{\log n}{\Delta_i}} .\]
It can be even shown that for arms $i$ with $\Delta_i > 0$, 
\[ T_i(n) = \bigO{\frac{\log (n/\delta)}{\Delta_i^2}} \]
with probability $1-\delta$ when $n$ is known to the player.

\subsection{Adversarial multi-armed bandit problems}

In adversarial bandit problems the rewards are selected by an adversary. If this is done beforehand (before the player interacts with the environment), then the adversary is called {\em oblivious} as the selection of rewards is independent from the arms~$I_t$ chosen by the player. In this case any upper bound on the pseudo-regret that holds for any selection of rewards is also an upper bound on the expected regret. 

If the selection of rewards $x_i(t)$, $1 \leq i \leq K$, depends on which arms $I_1,\ldots,I_{t-1}$ the player has chosen in the past, then the adversary is called {\em adaptive}. In this case a bound on the pseudo-regret does not necessarily translate into a bound on the expected regret.
Nevertheless, strong bounds on the regret against an adaptive adversary are known for the \expp\ algorithm~\citep{EXP3}:
\begin{theorem}[{\citealp[Theorem 3.3]{BC12}}] \label{t:EXP3}
When \expp\ is run with appropriate parameters depending on $n$, $K$, and $\delta$, then with probability $1-\delta$ its regret satisfies 
\[ R_\ada(n) = \bigO{\sqrt{nK\log(K/\delta)}} . \] 
\end{theorem}

\subsection{Results}

First, we state our lower bounds for oblivious and adaptive adversaries.
\begin{theorem} \label{theorem_lower_bound}
	Let $\alpha < 1$, $\epsilon > 0$, $\beta < 2$, and $\clower > 0$. Consider a player that achieves pseudo-regret 
	\[ \pseudoR_\sto(n) \leq \clower (\log n)^\beta \] 
	for any stochastic bandit problem with two arms and gap~$\Delta=1/8$. 
	%Let $n$ be such that  $(\log n)^{2-\beta} \geq \frac{64\clower \log 3}{(1-\alpha)\epsilon}$. 
	Then for large enough $n$ there is an adversarial bandit problem with two arms and an oblivious adversary such that the player suffers regret 
	\[ R_\obl(n) \geq n^\alpha/8 - 4\sqrt{n \log n} \]
	with probability at least $1/(16n^{\epsilon}) - 2/n^2$. Furthermore, there is an adversarial bandit problem with two arms and an adaptive adversary such that the player suffers expected regret 
	\[ \EE{R_\ada(n)} \geq \frac{n^{\alpha-\epsilon}}{128} - 3\sqrt{n \log n} . \]  
\end{theorem}

In Section~\ref{s:sao} we present our \sapo~algorithm (Stochastic and Adversarial Pseudo-Optimal) that achieves optimal pseudo-regret in stochastic bandit problems and nearly optimal pseudo-regret in adversarial bandit problems. Its performance is summarized in the following theorem.
\begin{theorem}\label{thm:main}
	For large enough $n$ and any $\delta > 0$, algorithm \sapo\ achieves the following bounds for suitable constants $C_\sto$, $C_\adv$, and $\Cib$:
	\begin{itemize} 
	\item For stochastic bandit problems with gaps $\Delta_i$ 
	such that\\ 
	$\Cib\sum_{i:\Delta_i > 0} \frac{\log (n/\delta)}{\Delta_i} \leq \sqrt{nK \log (n/\delta)}$, 
	\[ T_i(n) \leq C_\sto \frac{\log (n/\delta)}{\Delta_i^2} \]
	with probability $1-\delta$ for any arm $i$ with $\Delta_i > 0$, and thus
	\[ \pseudoR_\sto(n) \leq C_\sto \sum_{i:\Delta_i > 0} \frac{\log (n/\delta)}{\Delta_i} + \delta n . \]
	\item For adversarial bandit problems 
	\[ \pseudoR_\ada(n) \leq C_\adv K \sqrt{n \log (n/\delta)} + \delta n . \]
	\end{itemize}
\end{theorem}
\begin{remark}
	Our bound for adversarial bandit problems shows a worse dependency on $K$ than Theorem~\ref{t:EXP3}. This is an artifact of our current analysis and can be improved to a bound 
	$\pseudoR_\ada(n) = \bigO{\sqrt{nK\log(n/\delta)}}$. 
	%from Theorem~\ref{t:EXP3} is better in the logarithmic factor than our bound on $\pseudoR_\ada(n)$, note that typically 
	%$\log (1/\delta) = \Omega\left(\log n\right)$ such that the bounds are essentially equivalent.  
\end{remark}

\subsection{Comparison with related work}

\cite{BS12} show for their SAO algorithm that with probability $1-\delta$, 
\[ \sum_{i=1}^K \Delta_i T_i(n) \leq \bigO{\frac{K \log K (\log n/\delta)^2}{\Delta}} \]
for stochastic bandits where $\Delta = \min_{i:\Delta_i > 0}\Delta_i$, and 
\[ R_\ada(n) \leq \bigO{(\log K)(\log n)\sqrt{nK\log n/\delta}} \]
for adaptive adversarial bandits.
While our bounds in Theorem~\ref{thm:main} are somewhat tighter, in particular showing the optimal dependency on the gaps~$\Delta_i$ for stochastic bandits, we have only a result on the pseudo-regret for adversarial bandits. We conjecture though, that our analysis can be used to construct an algorithm that with probability $1-\delta$ achieves 
$T_i(n) \leq \bigO{{(\log n/\delta)^2}/{\Delta_i^2}}$ for stochastic bandits and 
$R_\ada(n) \leq \bigO{(\log K)(\log n)\sqrt{nK\log n/\delta}}$ for adaptive adversarial bandits.
%TODO

Our \sapo\ algorithm follows the general strategy of the SAO algorithm by essentially employing an algorithm for stochastic bandit problems that is equipped with additional tests to detect non-stochastic arms. 
A different approach is taken in~\citep{SS14}: here the starting point is an algorithm for adversarial bandit problems that is modified by adding an additional exploration parameter to achieve also low pseudo-regret in stochastic bandit problems. While this approach has not yet allowed for the tight $\bigO{\log n}$ regret bound in stochastic bandit problems (they achieve a $\bigO{\log^3 n}$ bound), the approach is quite flexible and more generally applicable than the SAO and \sapo\ algorithms.

\subsection{Proof sketch of the lower bound (Theorem~\ref{theorem_lower_bound})}

We present here the main idea of the proof. The proof itself is given in Appendix~\ref{app:fullproof-lower}.

We consider a stochastic bandit problem with constant reward $x_1(t)=1/2$ for arm~1 and Bernoulli rewards with $\mu_2=1/2-\Delta$ for arm~2, $\Delta=1/8$. We divide the time steps into phases of increasing length $L_j=3^j n^{\alpha}$, $j=0,\ldots,J$ with $J = \Omega(\log n)$. Since the pseudo-regret of the player is $\bigO{(\log n)^\beta}$, there is a phase $j^*$ where the expected number of plays of arm~2 in this phase is  $\bigO{(\log n)^{\beta-1}}$.

We construct an oblivious adversarial bandit by modifying the Bernoulli distribution of arm~2 in phase~$j^*$ and beyond by setting $\mu_2=1/2+\Delta$. By this modification arm~2 gives larger total reward than arm~1.

Because of the limited number of plays in phase~$j^*$, a standard argument shows that the player will not detect this modification during phase~$j^*$ with probability 
$\exp\{-O(\log^{\beta-1}\T)\}=\Omega(n^{-\epsilon})$. When the modification is not detected during phase~$j^*$, then in this phase the player suffers roughly regret $\Delta L_{j^*}$ against arm~2. This is not compensated by negative regret against arm~2 in previous phases since 
$\Delta \sum_{j=0}^{j^*-1}L_j \leq \Delta L_{j^*}/2$. Thus in this case the overall regret of the player against arm~2 is roughly $\Delta L_{j^*}/2 = \Omega(n^\alpha)$.

In a very similar way we can construct also an adaptive adversarial bandit: As for the oblivious bandit, we set $\mu_2 = 1/2+\Delta$ in phase~$j^*$. If the player chooses arm~2 only  $C(\log n)^{\beta-1}$ times in phase~$j^*$, then we keep $\mu_2 = 1/2+\Delta$ also for the remaining phases. As for the oblivious bandit this happens with probability $\Omega(n^{-\epsilon})$ and gives regret $\Omega(n^\alpha)$. To avoid negative regret, we switch back to $\mu_2=1/2-\Delta$, as soon as there more than $C(\log n)^{\beta-1}$ plays of arm~2 in phase~$j^*$. In this case the reward of the algorithm is roughly $n/2 + C\Delta(\log n)^{\beta-1}$ such that in this case 
$R(n) \geq -C\Delta(\log n)^{\beta-1}$. Hence the expected regret is 
$\EE{R(n)} \geq \Omega(n^{\alpha-\epsilon}) - C\Delta(\log n)^{\beta-1} = \Omega(n^{\alpha-\epsilon})$.

\section{The \sapo\  algorithm} 
\label{s:sao}

\begin{algorithm}[t]
		\caption{: \sapo\label{algo_K}}
		%\BlankLine
		\textbf{Input:} Number of arms $K$, number of rounds $n \geq K$, and confidence parameter $\delta$. 
		
		\medskip
		\textbf{Initialization:} All arms are active, $\cA(0) = \{1,\ldots,K\}$, 
		         $\cB(0) = \emptyset$.
	    
		\medskip
		\textbf{For} $t=1,\ldots,n$:
		\begin{enumerate}
			
			\item 
			\begin{enumerate}
				\item If there is an arm $i \in \cA(t-1)$ with $\bmu_i(t-1) \not\in [\blcb_i(t-1),\bucb_i(t-1)]$, \\then switch to \expp.
				\item If $\sum_{s=1}^{t-1} [\lcb^*(s) - x_{I_s}(s)] > \Cib\sqrt{Kn\log (n/\delta)}$, then switch to Exp3.P.
			\end{enumerate}
			
			\item Evict arms from $\cA$:
			\begin{enumerate}
				\item Let $B(t) = \{i \in \cA(t-1): \begin{array}[t]{l}
							  T_i(t-1) \geq \Ci\cdot\log(n/\delta) \;\wedge \\
				              \hmu_i(t-1) + \Cgap\cdot\width_i(t-1) < \lcb^*(t-1)\},
				              \end{array}$\\[0.5ex]
				   $\cA(t) = \cA(t-1) \setminus B(t)$, $\cB(t) = \cB(t-1) \cup B(t)$.
				\item For all $i \in B(t)$ set $\emu_i = \hmu_i(t-1)$,
					$\egp_i  = \Cgap\cdot\width_i(t-1)$, 
					\\
					$n_i(t)=t$, 
					$L_i(t) = L_i^0:= \lceil\Cpp K/\egp_i^2\rceil$, and $E_i(t)=0$.
			\end{enumerate}
			
			\item Choose $I_t=i$ with probabilities 
			\begin{equation*}
			p_i(t)= \left\{
			\begin{array}{cl}
			L_i^0/(K L_i(t)) & \textrm{for } i\in \cB(t) \\
			\spa{1-\sum_{j\in \cB(t)}p_j(t)} \Big/ |\cA(t)| & \textrm{for } i\in \cA(t)
			\end{array}
			\right.
			\end{equation*}
			
			\item Test and update all arms $i \in \cB(t)$: 
			\begin{enumerate}
				\item If $\exists s:n_i(t) \leq s \leq t: \hD_i(s,t) \geq \Caa \egp_i L_{i}(t) p_{i}(t)$, \label{step:detection}
				\item then $n_i(t+1)=t+1$, $L_i(t+1) = \max\{L_i(t)/2,L_i^0\}$,
				\\ \mbox{}\hphantom{then} and $E_i(t+1)=E_i(t)+1$,
				\item \mbox{}\hphantom{then} if $E_{i}(t+1) =E^0:=\lceil\ci\cdot\log (n/\delta)\rceil$, then switch to \textsc{Exp3.P};	
				\item else if $t=n_i(t)+L_i(t)-1$ then $n_i(t+1)=t+1$, $L_i(t+1)=2L_i(t)$, 
				\\\mbox{}\hphantom{then} and $E_i(t+1)=E_i(t)$;
				\item else $n_i(t+1)=n_i(t)$, $L_i(t+1) = L_i(t)$, and $E_i(t+1)=E_i(t)$.
			\end{enumerate}	
		
		\end{enumerate}
	
	\end{algorithm}

In its core the algorithm is an elimination procedure for stochastic bandits that is augmented by tests safeguarding against non-stochastic arms. If there is sufficient evidence for non-stochastic arms, then the algorithm switches to the adversarial bandit algorithm \expp, starting with the current time step.

The algorithm maintains a set of active arms~$\cA$ and a set of supposedly suboptimal ``bad'' arms~$\cB$. For each arm~$i$ it maintains the sample mean~$\hmu_i(s)$, 
\begin{align*}
	\hmu_i(s) &= \frac{1}{T_i(s)}\sum_{t=1}^{s} x_i(t) \II{I_t=i} , \\
	T_i(s) &= \sum_{t=1}^{s} \II{I_t=i} ,
\end{align*}
and also an unbiased estimate to deal with non-stochastic arms,
\begin{align*}
	\bmu_i(s) &= \frac{1}{s}\sum_{t=1}^{s} x_i(t) \frac{\II{I_t=i}}{p_i(t)} ,
\end{align*}
where $p_i(t)$ is the probability of choosing arm~$i$ at time~$t$. 
Confidence bounds\footnote{%
	We start with $\lcb_i(0)=\blcb_i(0)=0$ and $\bucb_i(0)=1$.} 
around the estimated means are used to evict arms from the active set~$\cA$, 
\begin{align*}
	\lcb_i(s) &= \max\{\lcb_i(s-1),\hmu_i(s) - \width_i(s)\}, \\
	\blcb_i(s) &= \max\{\blcb_i(s-1),\bmu_i(s)-\bwidth(s)\}, \\
	\bucb_i(s) &= \min\{\bucb_i(s-1),\bmu_i(s)+\bwidth(s)\}, \\  
	\lcb^*(s) &= \max_{1 \leq i \leq K} \max\{\lcb_i(s),\blcb_i(s)\}, \\
	\width_i(s) &= \sqrt{\Cwid \log (n/\delta)/T_i(s)} ,\\
	\bwidth(s) &= \sqrt{\Cwid K\log (n/\delta)/s} .
\end{align*}
Note that $\lcb_i(s)$, $\blcb_i(s)$, and $\lcb^*(s)$ are non-decreasing and $\bucb_i(s)$ are non-increasing. This reflects the intuition that confidence intervals should be shrinking and is used to safeguard against non-stochastic arms. 

An arm~$i$ is evicted from~$\cA$ in {\bf Step~2.a}, if 
it has a sufficient number of plays ($\Ci\cdot\log (n/\delta)$) for reasonably accurate estimates, and if 
its sample mean~$\hmu_i(t-1)$ is significantly smaller than the optimal lower confidence bound~$\lcb^*(t-1)$.
The additional distance $\Cgap\cdot\width_i(t-1)$ is used to estimate the gap~$\Delta_i$. For evicted arms, in {\bf Step 2.b} an estimate for the gap $\egp_i$ and the current estimated mean are frozen, $\emu_i=\hmu_i(t-1)$. For stochastic bandits the accuracy of this estimate is proportional to the estimated gap~$\egp_i$. These quantities are used in the tests for detecting non-stochastic arms. Also the starting time~$n_i(t)$ and the length~$L_i(t)=L_i^0$ of the first testing phase (see below), as well as the number of detections $E_i(t)=0$ are set. 

Since \sapo\ needs to perform well also against adversaries, all choices of arms are randomized.
In {\bf Step 3} an active arm is chosen uniformly at random, or with some smaller probability%
%\footnote{Since $L_i(t) \geq L_i^0$, $\sum_{i \in \cB(t)} p_i(t) \leq |\cB(t)|/K\leq 1/2$.} 
a bad arm~$i$ is chosen where the probability depends on the length of its current testing phase~$L_i(t)$. 
Choosing also bad arms is necessary to detect non-stochastic arms among the bad arms.

\subsection{Tests for detecting non-stochastic arms}

The most important test is in {Step 4.a} for detecting that a bad arm receives larger rewards than it should if it were stochastic. Such an arm could be optimal if the bandit problem is adversarial. The best way to view this test is by dividing the time steps of an evicted arm~$i$ into testing phases 
\[
   \tau_{i,1},\ldots,\tau_{i,2}-1;
   \tau_{i,2},\ldots,\tau_{i,3}-1;
   \tau_{i,3},\ldots,\tau_{i,4}-1;\ldots
\]
The first phase starts when arm $i$ is evicted from $\cA$. A phase~$k$ ends at time $\tau_{i,k+1}-1$ if either the phase has exhausted its length (Step~4.d), or when the test in Step~4.a reports a detection.\footnote{%
	The last phase ends when the total number of time steps~$n$ is exhausted or when the algorithm switches to \expp.} 
Thus the length parameter~$L_i(t)$ is only the maximal length of a phase and the phase may end earlier. In the notation of the algorithm $n_i(t)$ denotes the start of the current phase. Within a phase the probability $p_i(t)$ for choosing arm~$i$ is constant since the length parameter~$L_i(t)$ does not change (Step 4.e). For notational convenience we denote by $p_{ik}$ the probability for choosing arm~$i$ in its $k$-th testing phase, and by $L_{ik}$ the corresponding length parameter,
\begin{align*}
	 p_i(t) &= p_{ik} \text{~~for $i \in \cB(t)$ and $\tau_{i,k} \leq t < \tau_{i,k+1}$},
	 \\
	 L_i(t) &= L_{ik} \text{~~for $i \in \cB(t)$ and $\tau_{i,k} \leq t < \tau_{i,k+1}$},
	 \\
	 n_i(t) &= \tau_{i,k} \text{~~for $i \in \cB(t)$ and $\tau_{i,k} \leq t < \tau_{i,k+1}$}.
\end{align*}

Now the test in {\bf Step 4.a} checks if a bad arm $i$ has received significantly more rewards in the current phase then expected, given the estimated mean~$\emu_i$, the maximal phase length~$L_i(t)$ and the probability for choosing arm~$i$, $p_i(t)$, where
\[ \hD_i(s_1,s_2) = \sum_{t=s_1}^{s_2} [x_i(t) - \emu_i] \II{I_t=i} .
\] 
If arm $i$ is stochastic, then  $\EE{\hD_i(s_1,s_2)} =\bigO{L_i(t) \egp_i p_i(t)}$ such that a positive test suggests that the arm is non-stochastic. Since the expected number of plays of arm~$i$ is $L_i^0/K$ in each phase, the test is weak, though, with constant false positive and false negative rates. To avoid incorrectly classifying a stochastic arm as non-stochastic, the test is repeated several times. To make the tests independent, a new phase is started in {\bf Step~4.b} after a detection is reported. To avoid that too much regret accumulates in the case of a non-stochastic arm, the phase length is halved. If there have been~$E^0$  independent detections, then in {\bf Step~4.c} there is sufficient evidence for a non-stochastic arm and the algorithm switches to \expp. 

In {\bf Step~4.d} the phase ends because it has exhausted its length. Since the test in Step~4.a has given no detection, arm~$i$ has performed as expected and the algorithm has accumulated negative regret against this bad arm. This negative regret allows to start the next phase with a doubled phase length, even if the arm were non-stochastic. Doubling the phase length is necessary to avoid too many phases for a stochastic arm. (Remember that the expected number of plays of a bad arm is $L_i^0/K$ in each phase.)

In {\bf Step 4.e} none of the above condition is satisfied and the phase continues.

Additional simpler tests for non-stochastic arms are performed in Step~1. {\bf Step~1.a} checks whether for all active arms the unbiased estimates of the means obey the corresponding confidence intervals. 
Finally, {\bf Step 1.b} checks if the algorithm receives significantly less reward than expected from the best lower confidence bound. This may happen if a non-stochastic arm first appears close to optimal but then receives less rewards.

\subsection{Choice of constants in the \sapo\ algorithm}
\label{s:constants}

In the algorithm we keep the constant names because we find them easier to read than actual values. Proper values for the constants are as follows: $\Cwid=16$, $\Cib=522$, $\Ci=100/9$, $\Cgap=60$, $\Cpp=1300$, $\Caa=1/10$, and $\ci=15$.

%All unspecified constants need to be sufficiently large with the exception of $\Caa$ which needs to be sufficiently small.

\section{Preliminaries for the analysis of \sapo}\label{sec_pre}

An important tool for our analysis are concentration inequalities, in particular Bernstein's inequality for martingales and a variant of Hoeffding-Azuma's inequality for the maximum of partial sums, $\max_{1\leq s \leq t \leq n} \sum_{i=s}^t Y_i$. These inequalities are given in Appendix~\ref{app:concentration}.
We denote by $\hist_t$ the past up to and including time~$t$. 

The next lemma states some properties of algorithm \sapo. 
Let
 \[ T_i(s_1,s_2) = \#\{t:s_1 \leq t \leq s_2: I_t=i\} \]
 denote the number of plays of arm $i$ in time steps $s_1$ to $s_2$, let $\nbi$ be the time when arm~$i$ is evicted from~$\cA$,
 \[ i \in \cA(\nbi-1) \text{~~and~~} i \in \cB(\nbi), \]
 and let $\ns$ be the time step when \sapo\ switches to \expp.
 If \sapo\ never switches to \expp, then $\ns=n$.
\begin{lemma} \label{l:prelim}
	%With probability $1-\delta/2$ the following holds for all time steps $1\leq s \leq t \leq n$ and all arms $i$, $i'$:
	\begin{enumerate}[label={\rm(\alph*)},ref=\alph*,nosep]
		\iffalse
		\item \label{item:diffTi}
		If $T_{i}(s,t) \geq \Ci \log (n/\delta)$ and $t < \nbi$, then $T_{i}(s,t) \geq \cv T_{i'}(s,t)$.
		\item \label{item:initT}
		If $T_{i}(s,t) \geq \Ci \log (n/\delta)$ and $t < \nbi$, then $T_{i}(s,t) \geq (t-s+1)/(2K)$.
		\fi
		\item \label{item:emu-lcb} If $i \in \cB(t)$ then $\emu_i + \egp_i < \lcb^*(t)$.
		\item \label{item:numPhases}
		For each arm the number of testing phases $k$, $\tau_{i,k}\cdots\tau_{i,k+1}-1$  is \\at most  
		$\numPH = \lceil \log_2 n\rceil + 2E^0$.
		\item \label{item:TiB}
		With probability $1-\bigO{\delta}$, the number of plays of any bad arm $i$ is bounded as\\ 
		$T_i(\nbi,\ns) \leq \frac{101}{100} L_i^0\numPH/K = \bigO{\numPH/\egp_i^2}$.
	\end{enumerate}
\end{lemma}	 
\newcommand{\refP}[1]{\ref{l:prelim}\ref{item:#1}}
\begin{proof}{\bf (Sketch)} 
%
\iffalse	
Statement (\ref{item:diffTi}) and (\ref{item:initT}) follow from the fact that active arms are chosen uniformly, with probability at least $1/K$, and with higher probability than bad arms. Thus the probability that an active arm is chosen significantly less often than any other arm is small.
\fi
%
Statement (\ref{item:emu-lcb}) follows immediately from Step~2 of the algorithm since $\emu_i=\hmu_i(\nbi-1)$, $\egp_i=\Cgap\cdot\width_i(\nbi-1)$, $\hmu_i(\nbi-1)+\Cgap\cdot\width_i(\nbi-1) < \lcb^*(\nbi-1)$, and $\lcb^*(t)$ is non-decreasing. 
	
Statement (\ref{item:numPhases}) follows from the fact that Step~4.b (where the phase length is halved) is executed at most~$E^0$ times. In the other phases the phase length is doubled in Step~4.d. Since the phase length is at most $n$, the number of phases is at most $\log_2 n + 2E^0$.

For statement (\ref{item:TiB}) we observe that by the definition of $p_i(t)$ the expected number of plays in any testing phase of a bad arm~$i$ is $L_i^0/K$. Thus the expected number of plays in all phases is~$L_i^0 M/K$. Since the variance is bounded by the same quantity, an application of Bernstein's inequality gives the result. 

Detailed proofs are given in Appendix~\ref{app:prelim}.   
\end{proof}

\section{Analysis of \sapo\ for adversarial bandits}\label{sec_adv_analysis_K}

In this section we prove pseudo-regret bounds for \sapo\  against adversarial and possibly adaptive bandits.
Since we know from Theorem~\ref{t:EXP3} that \expp\ suffers small regret, we only need to bound the pseudo-regret of \sapo\  before it switches to \expp. 
For the remaining section we fix some arm~$i$.
We have
\begin{eqnarray}
	 \lefteqn{\sum_{t=1}^\ns x_i(t) - \sum_{t=1}^\ns x_{I_t}(t) 
	 \;=\;  {\sum_{t=1}^{\ns} \left[ x_i(t)-\lcb^*(t) \right] } 
	       +  {\sum_{t=1}^{\ns} \left[ \lcb^*(t) - x_{I_t}(t) \right] }} \nonumber  \\
	  &= &  {\sum_{t=1}^{\nbi-1} \left[ x_i(t)-\lcb^*(t) \right] } 
	       + {\sum_{t=\nbi}^{\ns} \left[ x_i(t)-\lcb^*(t) \right] } 
	         +  {\sum_{t=1}^{\ns} \left[ \lcb^*(t) - x_{I_t}(t) \right] } \label{eq:adv-Ri}
\end{eqnarray} 
The first sum in (\ref{eq:adv-Ri}) bounds the regret for the time when $i$ is an active arm. For stochastic arms, the best lower confidence bound $\lcb^*(t)$ would be not too far from the rewards of the arms that are still active. For non-stochastic arms, though, we need the tests in \sapo, in particular those in Step~1, to guarantee a similar behavior and achieve 
\begin{align}\label{eq:bound-active}
	\EE{\sum_{t=1}^{\nbi-1} \left[ x_i(t)-\lcb^*(t) \right]} 
	= \bigO{\sqrt{Kn\log (n/\delta)}} ,
\end{align}
see Appendix~\ref{app:bound-active}.

The crucial part of the analysis concerns the second sum in  (\ref{eq:adv-Ri}) which bounds the regret for the time when $i$ is a bad arm. For its analysis we explicitly track negative regret to compensate for positive regret. In Section~\ref{s:bad} below we sketch the main ideas for handling this sum (formal proofs are given in Appendix~\ref{app:bound-bad}), showing that
\begin{align}\label{eq:bound-bad}
	\EE{\sum_{t=\nbi}^{\ns} \left[ x_i(t)-\lcb^*(t) \right]} 
	= \bigO{\frac{K\log (n/\delta)}{\egp_i}} .
\end{align}
Note that $1/\egp_i=\bigO{\width_i(\nbi-1)} = \bigO{\sqrt{T_i(\nbi)/\log(n/\delta)}} 
= \bigO{\sqrt{n/\log(n/\delta)}}$ such that $\bigO{K\log (n/\delta)/\egp_i} = \bigO{K\sqrt{n\log(n/\delta)}}$.

Finally, the third sum can be observed by the algorithm and is taken care of by the test in Step~1.b, such that 
\begin{align}\label{eq:bound-1c}
	\sum_{t=1}^{\ns} \left[ \lcb^*(t) - x_{I_t}(t) \right]  
	= \bigO{\sqrt{Kn\log(n/\delta)}}.
\end{align}
Together, inequalities (\ref{eq:adv-Ri})--(\ref{eq:bound-1c}) and the bound on \expp\ in Theorem~\ref{t:EXP3} give the bound on the pseudo-regret in Theorem~\ref{thm:main}.

\subsection{Bounding the regret for bad arms}
\label{s:bad}

If a bad arm is non-stochastic, then it may first appear suboptimal but still be optimal after all. We need to show that the tests of our algorithm, in particular the test in Step~4.a, are sufficient to detect such a situation. Since the algorithm checks arms in $\cB(t)$ only rarely, it will take some time for such detection. In our analysis we explicitly compensate the regret during this delayed detection by the negative regret accumulated while arm~$i$ was performing suboptimally.  

We consider the testing phases $k$, $\tau_{i,k} \ldots \tau_{i,k+1}-1$, of arm~$i$, and recall that~$L_{ik}$ is the length parameter for phase~$k$ and $p_{ik}=L_i^0/(K L_{ik})$ is the probability for choosing arm~$i$ in phase~$k$. Furthermore, let~$E_{ik}$ the value of~$E_i(t)$ in phase~$k$.
Note that these quantities may change only when a new phase begins. 
We denote by $\PPik{\cdot}=\pr{\cdot|\hist_{\tau_{i,k}-1}}$ and $\EEik{\cdot}=\EE{\cdot|\hist_{\tau_{i,k}-1}}$ the probabilities and expectations conditioned on the past before phase~$k$. 

For any phase we have
	\begin{align} 
	\sum_{t=\tau_{i,k}}^{\tau_{i,k+1}-1} \left[ x_i(t)-\lcb^*(t) \right] 
	&= \sum_{t=\tau_{i,k}}^{\tau_{i,k+1}-1} \left[ x_i(t)-\emu_i+\emu_i-\lcb^*(t) \right] \nonumber
	\\&< \sum_{t=\tau_{i,k}}^{\tau_{i,k+1}-1} \left[ x_i(t)-\emu_i\right] - \egp_i[\tau_{i,k+1} - \tau_{i,k}]  \label{eq:negRegret}
	\end{align}  
by Lemma~\refP{emu-lcb}.
Thus we want to prevent that the rewards of arm~$i$ are significantly larger than the estimated mean~$\emu_i$. In particular, the test in Step~4.a is supposed  to detect events 
$D_i(s_1,s_2) > \CVV \egp_i L_{ik}$ with  
\[
D_i(s_1,s_2) := \sum_{t=s_1}^{s_2} [x_i(t) - \emu_i]  .  
\]
Since on average arm~$i$ is chosen only $L_i^0/K$ times per phase, there is a constant false negative rate $\qadv$ for missing such events. For appropriate $\Cpp$, though, the false negative rate $\qadv$ is sufficiently small, $\qadv \leq 1/25$: Since $\EE{\hD_i(s_1,s_2)} = p_{ik} D_i(s_1,s_2)$ for $\tau_{i,k} \leq s_1 \leq s_2 < \tau_{i,k+1}$, and Step~4.a tests for $\hD_i(s_1,s_2) > \Caa \egp_i L_{ik} p_{ik}$, we can bound $\qadv$ by Bernstein's inequality using that $1 \leq \egp_i^2 L_i^0/(K \Cpp)$ and a bound on the variance,
\[ \VV{\hD_i(s_1,s_2)} \leq L_{ik} p_{ik} = L_i^0/K \leq (\egp_i L_i^0/K)^2/\Cpp = (\egp_i L_{ik} p_{ik})^2/\Cpp  . \]  
The formal proof is given in Lemma~\ref{l:qadv}.  

We use the false negative rate $\qadv$ to bound $\EEik{D_i(\tau_{i.k},\tau_{i.k+1}-1)}$. Each time an event $D_i(s,t) > \CVV \egp_i L_{ik}$ is missed (we consider only non-overlapping such events), $D_i(\tau_{i.k},t)$ has increased by at most $\CVV \egp_i L_{ik} + 1$, and the probability for the $m$-th miss is at most $\qad{m}$. When such an event is detected, then the phase ends and  $D_i(\tau_{i.k},t)$ again has increased by at most $\CVV \egp_i L_{ik} + 1$.
Thus (see  Lemma~\ref{l:expD} for the formal proof)
\begin{align*}
	\EEik{D_i(\tau_{i.k},\tau_{i.k+1}-1)} \leq (\CVV \egp_i L_{ik}+1) \sum_{m \geq 0} \qad{m} = \frac{\CVV \egp_i L_{ik}+1}{1-\qadv}  
\end{align*}
which by (\ref{eq:negRegret}) gives 
\begin{align}
	\EEik{\sum_{t=\tau_{i,k}}^{\tau_{i,k+1}-1} \left[ x_i(t)-\lcb^*(t) \right]}
	%&<\EEik{\sum_{t=\tau_{i,k}}^{\tau_{i,k+1}-1} \left[ x_i(t)-\emu_i\right] - \egp_i[\tau_{i,k+1} - \tau_{i,k}]}
	%\nonumber
	%\\&=\EEik{D_i(\tau_{i.k},\tau_{i.k+1}-1)  - \egp_i[\tau_{i,k+1} - \tau_{i,k}]} \nonumber
	&< \frac{\CVV \egp_i L_{ik}+1}{1-\qadv} - \egp_i \EEik{\tau_{i,k+1} - \tau_{i,k} }.  \label{eq:ED}
\end{align}
Since the bound in (\ref{eq:ED}) is large for large $L_{ik}$, we show that such a large contribution to the regret can be compensated by negative regret in previous phases due  
to the term $- \egp_i[\tau_{i,k+1} - \tau_{i,k}]$. 
We show by backward induction over the phases that the expected regret starting from phase~$k$ can be bounded,
\begin{align*}
	\EEik{\sum_{t=\tau_{i,k}}^{\ns} \left[ x_i(t)-\lcb^*(t) \right]}
	\leq \Phi_i(k,L_{ik}) := L_{ik}\egp_i/2 + 3 L_i^0\egp_i(M-k+1)
\end{align*}
where $M$ is the maximal number of phases from Lemma~\refP{numPhases}.

\begin{lemma} \label{l:recursion}
	Let 
	\[ F_{ik} = \sum_{t=\tau_{i,k}}^{\ns} \left[ x_i(t)-\lcb^*(t) \right] . \]
	Then
	\begin{align*}
		\EEik{F_{ik}} \leq \Phi_i(k,L_{ik}) . 
	\end{align*}
\end{lemma}
\begin{proof}
	Let $k_S$ be the last phase before the algorithm switches to \expp\ with $\tau_{k_S+1}-1=\ns$. By Lemma~\ref{l:prelim}\ref{item:numPhases} we have $k_S \leq M$.
	For $k=k_S+1$ the lemma holds trivially since $F_{i,k_S+1}=0$. 
	
	By (\ref{eq:ED}) we have   
	\begin{align*}
		\EEik{F_{ik}} \leq \frac{\CVV \egp_i L_{ik} + 1}{1-\qadv} +
		 \EEik{F_{i,k+1} - \egp_i(\tau_{i,k+1} - \tau_{i,k}) } . 
	\end{align*}
	For the expectation on the right hand side we distinguish three cases, depending on the termination condition of phase $k$ and the value of $L_{ik}$.
	
	\smallskip \noindent
	\CASE{1}: Phase $k$ is terminated by the condition in Step 4.d.	
	Then $L_{i,k+1}=2L_{ik}$ and   
	\begin{align}
		\EEik{\left. F_{i,k+1}- \egp_i(\tau_{i,k+1} - \tau_{i,k})\right| \text{\CASE{1}} } 
		&\leq \Phi_i(k+1,2L_{ik})   - \egp_i L_{ik}  \label{eq:case1}
	\end{align}
	using the induction hypothesis.
	\\This is the case where negative regrets accumulate since~$\CVV/(1-\qadv) < 1$. 
	
	\smallskip \noindent
	\CASE{2}: Phase $k$ is terminated by the condition in Step 4.a (\footnote{%
		If $k$ is the last phase and the phase is terminated by a condition in Step~1, then the same analysis applies but the value of $L_{k+1,i}$ is irrelevant, since $F_{i,k+1}=0$.})  
	and $L_{ik} > L_i^0$.\\ 	
	Then $L_{i,k+1}=L_{ik}/2$ and 
	\begin{align}
		\EEik{\left. F_{i,k+1}- \egp_i(\tau_{i,k+1} - \tau_{i,k})\right| \text{\CASE{2}} } 
		&\leq \Phi_i(k+1,L_{ik}/2)    .
	\end{align}
	
	\smallskip \noindent
	\CASE{3}: Phase $k$ is terminated by the condition in Step 4.a 
	and $L_{ik}=L_i^0$.\\ 	
	Then $L_{i,k+1}=L_i^0$ and 
	\begin{align}
		\EEik{\left. F_{i,k+1}- \egp_i(\tau_{i,k+1} - \tau_{i,k})\right| \text{\CASE{3}} } 
		&\leq \Phi_i(k+1,L_i^0)   .  \label{eq:case3}
	\end{align}
To complete the induction proof, we need to show that for all three cases the right hand side of~(\ref{eq:case1})--(\ref{eq:case3}) is upper bounded by
\[
\Phi_i(k,L_{ik}) - \frac{\CVV \egp_i L_{ik} + 1}{1-\qadv} . 
\]	
This can be verified by straightforward calculation.
%, provided that $\qadv$ and $\Caa$ are small enough. 		
\end{proof}
Now (\ref{eq:bound-bad}) follows from Lemma~\ref{l:recursion} for $k=1$: 
	\[
		\EE{\sum_{t=\nbi}^{\ns} \left[ x_i(t)-\lcb^*(t) \right]} 
		\leq \Phi_i(1,L_i^0) = \bigO{L_i^0 \egp_i M} = \bigO{\frac{K \log (n/\delta)}{\egp_i}} .   
	\]	

\iffalse
By simple calculus we have the following inequalities for $\Phi(L,E)$:
\begin{align}
	\Phi(k,L) &\geq \frac{\CVV \egp_i L + 1}{1-q} + 1 - \egp_i L+ \Phi(k+1,2L) ,
	\label{eq:Phi-case1}
	\\
	\Phi(k,L) &\geq \frac{\CVV \egp_i L + 1}{1-q} + 1 + \Phi(k+1,L/2) 
	&&\text{for $L > \lceil \Cii/\egp_i^2 \rceil$,}
	\label{eq:Phi-case2L}
	\\
	\Phi(k,L) &\geq \frac{\CVV \egp_i L + 1}{1-q} + 1 + \Phi(k+1,L) 
	&&\text{for $L = \lceil \Cii/\egp_i^2 \rceil$,}
	\label{eq:Phi-case2E}
	%
\end{align} 
since 
\begin{align*}
	\frac{2\CVV}{1-q} \leq 1/2 \leq 1-\frac{\CVV}{1-q},
	\\
	3 \geq \frac{\CVV(\Cii+1)+1}{1-q} + 1.
\end{align*}
\fi

\section{The stochastic analysis}\label{sec_sto_analysis_K}

\iffalse
{\bf Needed?}\\
Denote the accumulated reward of $i$ from $t+1$ to $t'$ by $\GAitt{i}{t}{t'} = \sum_{\tau=t+1}^{t'}g_{I_\tau}(\tau)\Id{I_\tau=i}$. 
%Denote $\max_i\wid_i(t)$ by $\wid(t)$.
We further define the accumulated reward and the deviation of arm $i$ from $t$ to $t'$ as $\Gitt{i}{t}{t'} = \sum_{\tau=t+1}^{t'}g_i(\tau)$ and as $\DAitt{i}{t}{t'} = \sum_{\tau=t+1}^{t'}(g_{I_\tau}(\tau)-\emu_i)\Id{I_\tau=i}$, respectively.
\fi

In this section we assume that all arms $i$ are indeed stochastic with means $\mu_i$.
Recall that $\Delta_i = \mu^* - \mu_i$, $\mu^*=\max_i \mu_i$. We show that with high probability the algorithm does not switch to \expp\ and any suboptimal arm $i$ is chosen at most
$\bigO{\log(n/\delta)/\Delta_i^2}$ times. 

We already have from Lemma~\refP{TiB} that with probability $1-\bigO{\delta}$, 
$T_i(\nbi,\ns) = \bigO{\numPH/\egp_i^2}$ for all arms. Thus we only need to bound the number of plays before an arm is evicted from $\cA$, $T_i(1,\nbi-1)$.
The next lemma summarizes some properties of \sapo\ against stochastic bandits.
\begin{lemma} \label{l:sto-prelim}
	With probability $1-\bigO{\delta}$ the following holds for all time steps $t$ and all arms $i$:
	\begin{enumerate}[label={\rm(\alph*)},ref=\alph*,nosep]
		\item \label{item:wid-bar}
		If $i \in \cA(t)$ then $|\bmu_i(t) - \mu_i| \leq \bwidth(t)/2$.
		\item \label{item:wid}
		If $i \in \cA(t)$ then $|\hmu_i(t) - \mu_i| \leq \width_i(t)/2$.
		\item \label{item:ConfInt}
		If $i \in \cA(t)$ then $\bmu_i(t),\mu_i \in [\blcb_i(t),\bucb_i(t)]$ and 
		$\hmu_i(t),\mu_i \geq \lcb_i(t)$.
		\item \label{item:lcb*}
		If $\Delta_{i^*}=0$ then $i^* \in \cA(t)$. Furthermore, $\mu^* \geq \lcb^*(t)$.
		\item \label{item:DeltaB}
		If $i \in \cB(t)$ then $\egp_i \leq 2\Delta_i$.
	\end{enumerate}
\end{lemma}	 
\newcommand{\refSP}[1]{\ref{l:sto-prelim}\ref{item:#1}}
\begin{proof}{\bf (Sketch)} 
Statements (\ref{item:wid-bar}) and (\ref{item:wid}) follow from Hoeffding-Azuma's inequality.
Details are given in Appendix~\ref{app:sto-prelim}.   
	
For statement (\ref{item:ConfInt}) we observe that by construction there is a time $s \leq t$ with $\bmu_i(s) - \bwidth(s) = \blcb_i(t)$. Thus (\ref{item:wid-bar}) implies 
$\bmu_i(t) \geq \mu_i - \bwidth(t)/2 \geq \bmu_i(s) -\bwidth(s)/2 - \bwidth(t)/2 
\geq \bmu_i(s) -\bwidth(s) = \blcb_i(t)$. The other inequalities follow analogously. 

Statement (\ref{item:lcb*}) is proven by induction on $t$. 
Let $i^*$ be an arm with $\mu_{i^*}=\mu^*$.  
If $i^* \in \cA(t-1)$ then we have by (\ref{item:ConfInt}) that $\mu^* \geq \lcb^*(t-1)$.	 
If any arm~$i$ is evicted at time~$t$, then we have by Step~2.a	and~(\ref{item:wid}) that 
$\Delta_i = \mu^* - \mu_i \geq \lcb^*(t-1) - \hmu_i(t-1) - \width_i(t-1)/2 
\geq (\Cgap - 1/2) \width_i(t-1) > 0$. Thus $i \neq i^*$ and $i^* \in \cA(t)$. 

This also shows that when arm~$i$ is evicted, 
$\egp_i = \Cgap \cdot \width_i(t-1) \leq \Cgap/(\Cgap-1/2)\Delta_i$, which is  statement~(\ref{item:DeltaB}). 
\end{proof}
To get a bound on $T_i(1,\nbi-1)$, we show that $\egp_i=\Cgap\cdot\width_i(\nbi-1)$ cannot be too small.
\begin{lemma} \label{l:Delta>}
With probability $1-\bigO{\delta}$ it holds for all times $t$ and all arms~$i \in \cA(t)$ with $T_i(t-1) \geq \Ci \log(n/\delta)$, that
\[ \Cgap \cdot \width_i(t-1) \geq \Delta_i/2 . \]
\end{lemma}
The argument behind the lemma is that if $i \in \cA(t)$ then $\Cgap \cdot \width_i(t-1) \geq \lcb^*(t-1) - \hmu_i(t-1)$ where $\lcb^*(t-1)$ is sufficiently close to $\mu^*$ and $\hmu_i(t-1)$ is sufficiently close to $\mu_i$. The proof is given in Appendix~\ref{app:Delta>}.

Since $i \in \cA(\nbi-1)$, we get from Lemma~\ref{l:Delta>} that with probability $1-\bigO{\delta}$, 
\begin{align*}
T_i(\nbi-1) 
\leq  T_i(\nbi-2) +1 
=\frac{\Cwid \log(n/\delta)}{[\width_i(\nbi-2)]^2} + 1
\leq \frac{4 \Cwid \Cgap^2 \log(n/\delta)}{\Delta_i^2} +1 .
\end{align*}
Together with Lemma~\refP{TiB} we have with probability $1-\bigO{\delta}$ that for all arms,
\begin{align} \label{eq:Tns}
	T_i(\ns)
	\leq \frac{101}{100} L_i^0 M/K + \frac{4 \Cwid \Cgap^2 \log(n/\delta)}{\Delta_i^2} +1 
	= \bigO{\frac{\log(n/\delta)}{\Delta_i^2}} .
\end{align}

Finally, we need to bound the probability the \sapo\ switches to \expp.
Switching in Step 1.a is already handled by Lemma~\refSP{ConfInt}.
Switching in Step~1.b is also unlikely, since it would mean that the algorithm has accumulated large regret. This contradicts the upper bound~(\ref{eq:Tns}).
Lemma~\ref{l:step1b} shows that \sapo\ switches in Step~1.b only with probability $1-\bigO{\delta}$.

The difficult part, though, is to show that the condition in Step~4.a is not triggered too often such that Step~4.c switches to \expp. We first calculate the false positive rate $\qsto$, the probability that during a given phase the condition in Step~4.a is triggered. The false positive rate is again a constant but small, $\qsto \leq 0.21$, see Lemma~\ref{l:qs}.

Now for a fixed arm the probability that in exactly $E\geq E^0$ out of at most $M$ phases the condition in Step~4.a is triggered, is at most ${M \choose E}\qsto^E$. We set $p=\qsto/(1+\qsto)$ and use a tail bound for the binomial distribution to sum over $E=E^0,\ldots,M$:
	\begin{align*} 
	\sum_{E=E^0}^M {M \choose E}\qsto^E &=  (1+\qsto)^M \sum_{E=E^0}^M {M \choose E}p^E (1-p)^{M-E}
	\\&\leq (1+\qsto)^M \exp\left\{-M \cdot D(E^0/M || p)\right\} 
	\end{align*}
where $D(a||p) = a \log\frac{a}{p} + (1-a)\log\frac{1-a}{1-p}$ is the relative entropy. Since $\frac{E^0}{M} \geq \frac{\ci}{2\ci+1/\log 2}$, this sum is $\bigO{\delta/n}$ and a union bound over the arms completes the proof.
%for sufficiently small $\qsto$ and sufficiently large~$\ci$. 

%>>>>>>>>>>>>>>>>>>>>>>>>>>>>>>>>>>>>>>>>>>>>>>>>>>>>>>>>>>>>>
%>>>  Acknowledgments  >>>>>>>>>>>>>>>>>>>>>>>>>>>>>>>>>>>>>>>
%>>>>>>>>>>>>>>>>>>>>>>>>>>>>>>>>>>>>>>>>>>>>>>>>>>>>>>>>>>>>>
% Note: will not appear in anonymized version
\acks{We thank the anonymous reviewers for their very valuable comments. The research leading to these results has received funding from the European Community's Seventh
	Framework Programme (FP7/2007-2013) under grant agreement 
	n$^\circ$ 231495 (CompLACS) and from the Austrian Science Fund (FWF) under contract P~26219-N15.
}	

%>>>>>>>>>>>>>>>>>>>>>>>>>>>>>>>>>>>>>>>>>>>>>>>>>>>>>>>>>>>>>
%>>>  References  >>>>>>>>>>>>>>>>>>>>>>>>>>>>>>>>>>>>>>>>>>>>
%>>>>>>>>>>>>>>>>>>>>>>>>>>>>>>>>>>>>>>>>>>>>>>>>>>>>>>>>>>>>>
\bibliography{bib_sao}

%>>>>>>>>>>>>>>>>>>>>>>>>>>>>>>>>>>>>>>>>>>>>>>>>>>>>>>>>>>>>>
%>>>  Appendix  >>>>>>>>>>>>>>>>>>>>>>>>>>>>>>>>>>>>>>>>>>>>>>
%>>>>>>>>>>>>>>>>>>>>>>>>>>>>>>>>>>>>>>>>>>>>>>>>>>>>>>>>>>>>>
\appendix

\section{Concentration inequalities}\label{app:concentration}

%PA modified and made statements more explicit
%Fact~\ref{fact_concentration_martingale} is an extension of the Bernstein's inequality that holds for martingale difference sequences. Fact~\ref{fact_concentration_max} is an extended Hoeffding-Azuma's inequality that is strengthened to refer to maxima.

\begin{lemma}[{\cite[Theorem 3.15]{McD98}}]\label{l:bernstein}
	%[Extended Bernstein's inequality. Theorem 3.15 in \citep{McD98}.]  {Mark it later.} \\
	Let $Y_1, \ldots , Y_N$ be a martingale difference sequence with $S_N=Y_1+ \ldots +Y_N$ with the corresponding filtration $F_0 \subseteq F_1 \subseteq \ldots \subseteq F_N$. Let $Y_i \leq b$ and $\sum_{i=1}^N \expt{Y_i^2|F_{i-1}} \leq V$.
	Then for any $z \geq 0$,
	$$\pr{S_N \geq z} \leq \exp\spa{-z^2/(2V+2bz/3)}.$$
\end{lemma}

\begin{lemma}[{\cite[Theorem 3.13]{McD98}}]\label{l:max}
	%[Hoeffding-Azuma's extended inequality (Theorem 3.13) and its extension to maximum, \citep{McD98}.]  {Mark it later.} \\
	%[Hoeffding-Azuma's inequality. Theorem 3.10 in \citep{McD98}.] \\
	Let $Y_1, \ldots , Y_N$ be a martingale difference sequence with $a_k \leq Y_k \leq
	b_k$ for suitable constants $a_k$, $b_k$. Then for any $z\geq0$,
	\begin{eqnarray*}
		\pr{\max_{1 \leq m \leq N}\sum_{k=1}^m Y_k \geq z} \leq \exp\spa{-2z^2\left/\sum_{k=1}^N (b_k-a_k)^2 \right.}.
	\end{eqnarray*}
\end{lemma}

\begin{corollary} \label{cor:maxmax}
	Let $Y_1, \ldots , Y_N$ be a martingale difference sequence with $a_k \leq Y_k \leq
	b_k$ for suitable constants $a_k$, $b_k$. Then for any $z\geq0$,
	\begin{eqnarray*}
		\pr{\max_{1 \leq s \leq t \leq N}\sum_{k=s}^t Y_k \geq z} 
		\leq &2\exp\spa{-z^2\left/2\sideset{}{_{k=1}^N}\sum (b_k-a_k)^2 \right.}.
	\end{eqnarray*}
\end{corollary}
\begin{proof}
	\begin{eqnarray*}
		{\pr{\max_{1 \leq s \leq t \leq N}\sum_{k=s}^t Y_k \geq z}} 
		&\leq& \pr{\max_{1 \leq t \leq N}\sum_{k=1}^t Y_k \geq z/2} 
		+  \pr{\max_{1 \leq s \leq N}\sum_{k=1}^{s-1} (-Y_k) \geq z/2} 
		\\&\leq &2\exp\spa{-z^2\left/2\sideset{}{_{k=1}^N}\sum (b_k-a_k)^2 \right.}.
	\end{eqnarray*}
\end{proof}

\section{Proof of the lower bound (Theorem~\ref{theorem_lower_bound})}
\label{app:fullproof-lower}

\newcommand{\Tjs}{{T_2^{j^*}}}
\newcommand{\tj}{{t_{j^*}}}

Let $\Delta = 1/8$. We consider a stochastic bandit problem with constant reward $x_1(t)=1/2$ for arm~1 and Bernoulli rewards with $\mu_2=1/2-\Delta$ for arm~2. We divide the time steps into phases of increasing length $L_j=3^j \lfloor n^{\alpha} \rfloor$, $j=0,\ldots,J-1$ with $J \geq \frac{1-\alpha}{\log 3}\log n$ and an incomplete last phase $j=J$. Since the pseudo-regret of the player is at most $\clower(\log n)^\beta$, there is a phase $j^* < J$ where the expected number of plays of arm~2 in this phase is at most $B$ with 
\[ B = \frac{8\clower \log 3}{1-\alpha}(\log n)^{\beta-1} . \] 

We construct an adversarial bandit problem by modifying the Bernoulli distribution of arm~2. Before phase~$j^*$ the distribution remains unchanged with $\mu_2 = 1/2-\Delta$, but in phase~$j^*$ and beyond we set $\mu_2=1/2+\Delta$. Since this bandit problem depends only on the player strategy (for identifying phase~$j^*$) but not on the actual choices of the player, this adversary is oblivious.

Let $\Tjs$ be the number of plays of arm~2 in phase~$j^*$, and let $\Padv{\cdot}$ and $\Eadv{\cdot}$ denote the probability and expectation in respect to this adversarial bandit problem.
By Lemma~\ref{l:lower} below we have 
\[ \Padv{\Tjs \leq 4B} \geq 1/(16n^\epsilon) . \]
Since $x_{I_t}(t)-\Eadv{x_{I_t}(t)|\hist_{t-1}}$, $t=1,\ldots,n$, forms a martingale difference sequence, we can apply Azuma-Hoeffding's inequality (Lemma~\ref{l:max}) and obtain
\[ \Padv{\sum_{t=1}^n x_{I_t}(t) \geq \sum_{t=1}^n \Eadv{x_{I_t}|\hist_{t-1}} + \sqrt{2n \log n}} \leq 1/n^2  \]
and
\begin{align} \label{eq:lower-concentration}
	\Padv{\Tjs \leq 4B 
	\wedge \sum_{t=1}^n x_{I_t}(t) < \sum_{t=1}^n \Eadv{x_{I_t}|\hist_{t-1}} + \sqrt{2n \log n}} 
    \geq 1/(16n^\epsilon) - 1/n^2 .
\end{align}
By the construction of the adversarial bandit problem, $\Tjs \leq 4B$ implies that
\begin{align} \label{eq:lower-regret}
	\sum_{t=1}^n \Eadv{x_{I_t}|\hist_{t-1}} \leq n/2 + 4B\Delta + (n-\tj)\Delta , 
\end{align}
where $\tj$ denotes the time step at the end of phase~$j^*$. For arm~2 we have
\begin{align*} 
	\sum_{t=1}^n \Eadv{x_2(t)} 
	&= n/2 - \sum_{j=0}^{j^*-1}L_j \Delta + L_{j^*}\Delta + (n-\tj)\Delta  
	\\
	&= n/2 + \lfloor n^\alpha \rfloor \Delta \left(3^{j^*} - \frac{3^{j^*}-1}{2}\right) + (n-\tj)\Delta  
	\\
	&\geq n/2 + \lfloor n^\alpha \rfloor \Delta + (n-\tj)\Delta  .
\end{align*}
Applying Azuma-Hoeffdings's inequality for arm~2 and combining with~(\ref{eq:lower-concentration}) and~(\ref{eq:lower-regret}) we get
\[ \Padv{\sum_{t=1}^n x_2(t) - \sum_{t=1}^n x_{I_t}(t) 
	\geq  \lfloor n^\alpha \rfloor \Delta - 4B\Delta - 2\sqrt{2n \log n}}
	\geq 1/(16n^\epsilon) - 2/n^2 .
\]
By the condition on $n$, $4B\Delta \leq (\epsilon\log n)/(16 \Delta)$ such that 
$\lfloor n^\alpha \rfloor \Delta - 4B\Delta - 2\sqrt{2n \log n} \geq n^\alpha/8 - 4\sqrt{n \log n}$, which completes the proof of the high probability lower bound.

For the lower bound on the expected regret we construct an adaptive adversary by modifying the construction above: Let $\Tjs(t)$ be the number of plays of arm~2 in phase~$j^*$ up to and including time step~$t$. If $\Tjs=\Tjs(\tj) \leq 4B$ then the adversarial bandit problem above remains unmodified. If there is a time step~$t \leq \tj$ with $\Tjs(t) > 4B$, then for all time steps $>t$ we set again $\mu_2=1/2-\Delta$.

From the argument for the oblivious adversary we have 
\begin{eqnarray*} 
	\lefteqn{\EE{\left. \sum_{t=1}^n x_2(t) - \sum_{t=1}^n x_{I_t}(t) \right| \Tjs \leq 4B}
        \pr{\Tjs \leq 4B} }
    \\&\geq & \left[n^\alpha/8 - 4\sqrt{n \log n}\right]
         \pr{\Tjs \leq 4B 
         	    \wedge \sum_{t=1}^n x_2(t) - \sum_{t=1}^n x_{I_t}(t) \geq n^\alpha/8 - 4\sqrt{n \log n} }  
    \\&& - n \cdot \pr{\Tjs \leq 4B 
    	\wedge \sum_{t=1}^n x_2(t) - \sum_{t=1}^n x_{I_t}(t) < n^\alpha/8 - 4\sqrt{n \log n} }  
    \\&\geq & \left[n^\alpha/8 - 4\sqrt{n \log n}\right]\left[1/(16n^\epsilon) - 2/n^2\right] 
        -2/n 
    \\&\geq & \left[n^\alpha/8 - 4\sqrt{n \log n}\right]\frac{1}{(16n^\epsilon)} - 3 . 
\end{eqnarray*}
Analogously we get  
\begin{eqnarray*} 
 	\lefteqn{\EE{\left. \sum_{t=1}^n x_1(t) - \sum_{t=1}^n x_{I_t}(t) \right| \Tjs > 4B}
 	\pr{\Tjs > 4B} }
 	\\&\geq &-(4B+1)\Delta - \sqrt{2n \log n}  
 	\\&& - n \cdot \pr{\Tjs > 4B \wedge \sum_{t=1}^n x_1(t) - \sum_{t=1}^n x_{I_t}(t) < -(4B+1)\Delta - \sqrt{2n \log n}  }  
 	\\&\geq &-(4B+1)\Delta - \sqrt{2n \log n} -1/n 
 	\\&\geq &- 2\sqrt{n \log n} . 
\end{eqnarray*}
Thus
\begin{align*} 
	\EE{\max_i \sum_{t=1}^n x_i(t) - \sum_{t=1}^n x_{I_t}(t) }
	&\geq \left[n^\alpha/8 - 4\sqrt{n \log n}\right]\frac{1}{(16n^\epsilon)} 
	- 3 - 2\sqrt{n \log n} 
	\\&\geq \frac{n^{\alpha-\epsilon}}{128} - 3\sqrt{n \log n} .
\end{align*}

\begin{lemma} \label{l:lower} 
	For any $n$ with $(\log n)^{2-\beta} \geq \frac{64\clower \log 3}{(1-\alpha)\epsilon}$,
	 \[ \Padv{\Tjs \leq 4B} \geq 1/(16n^\epsilon) . \]
\end{lemma}
\newcommand{\Gjs}{{G_2^{j^*}}}
\begin{proof}
	The proof follows a standard argument, e.g.~\citep{MT04}.
	
	Let $\Psto{\cdot}$ and $\Esto{\cdot}$ denote the probability and expectation in respect to the stochastic bandit problem defined above.
	Since 
	$\Esto{\Tjs} \leq B$ 
	we have 
	$\Psto{\Tjs > 4B} < 1/4$ 
	and thus
	\begin{align} 
		\Psto{\Tjs \leq 4B} > 3/4 .
		\label{eq:Tjs}
	\end{align} 
	Let $\Gjs$ be the sum of rewards received when playing arm~2 in phase~$j^*$. Conditioned on $\Tjs$, $\Gjs$ is a binomial random variable with parameters $\Tjs$ and $\mu_2$. Hence by~\citep{KB80},
	\begin{align}  
		\Psto{\Gjs \leq \lfloor \Tjs(1/2-\Delta) \rfloor} \leq 1/2 . 
		\label{eq:Gjs}
	\end{align} 
	Let $\omega$ denote a particular realization of rewards $x_i(t)$, $i \in \{1,2\}$, $1 \leq t \leq \tj$, and player choices $I_1,\ldots,I_\tj$. 
	For any realization $\omega$ the probabilities $\Psto{\omega}$ and $\Padv{\omega}$ are related by
	\begin{align*} 
	\Padv{\omega} 
	&= \Psto{\omega}\frac{(1/2+\Delta)^{\Gjs(\omega)} (1/2-\Delta)^{\Tjs(\omega)-\Gjs(\omega)}}{%
		(1/2-\Delta)^{\Gjs(\omega)} (1/2+\Delta)^{\Tjs(\omega)-\Gjs(\omega)}} \\
	&= \Psto{\omega}\left(\frac{1-2\Delta}{1+2\Delta}\right)^{\Tjs(\omega)-2\Gjs(\omega)}
	. 
	\end{align*} 
	If $\Gjs(\omega) \geq \lfloor (1/2-\Delta)\Tjs(\omega) \rfloor$ then 
	\begin{align*} 
		\Padv{\omega} 
		&\geq \Psto{\omega}
		\left(\frac{1-2\Delta}{1+2\Delta}\right)^{\Tjs(\omega)-2((1/2-\Delta)\Tjs(\omega)-1)}
		\\
		&= \Psto{\omega}
		\left(\frac{1-2\Delta}{1+2\Delta}\right)^{2\Delta\Tjs(\omega)+2}
		. 
	\end{align*} 
	If furthermore $\Tjs(\omega) \leq 4B$, then 
	\begin{align*} 
		\Padv{\omega} 
		\geq \Psto{\omega}
		\left(\frac{1-2\Delta}{1+2\Delta}\right)^{8\Delta B + 2}
		. 
	\end{align*} 
	Hence
	\begin{align*} 
		\Padv{\Tjs \leq 4B} 
		&\geq \Padv{\Tjs \leq 4B \wedge \Gjs \geq \lfloor \Tjs(1/2-\Delta) \rfloor}  
		\\
		&\geq \Psto{\Tjs \leq 4B \wedge \Gjs \geq \lfloor \Tjs(1/2-\Delta) \rfloor}
		\left(\frac{1-2\Delta}{1+2\Delta}\right)^{8\Delta B + 2}
		\\
		&~\text{[by (\ref{eq:Tjs}) and (\ref{eq:Gjs})]} \\
		&\geq \frac{1}{4} \left(\frac{1-2\Delta}{1+2\Delta}\right)^{8\Delta B + 2} 
		\\
		&\geq \frac{1}{4} \left(1-4\Delta\right)^{8\Delta B + 2}   
		\\
		&~\text{[$\Delta = 1/8$, $1-x \geq e^{-2x}$ for $0 \leq x \leq 1/2$]} \\
		&\geq \frac{1}{16} \exp\{-64\Delta^2 B\}   
		\\
		&\geq \frac{1}{16 n^\epsilon}    
	\end{align*} 
	for $(\log n)^{2-\beta} \geq \frac{64\clower \log 3}{(1-\alpha)\epsilon}$.
	
\end{proof}

\section{Proof of Lemma~\ref{l:prelim}}
\label{app:prelim}

\iffalse
\begin{proof}{\bf of (\ref{item:diffTi})} We fix $s$, $t$, $i$, and $i'$. By the construction of \sapo\  we have 
	$\pr{I_t=i|\hist_{t-1},t < \nbi} \geq \pr{I_t=i'|\hist_{t-1},t < \nbi}$.
	From $I_{s},\ldots,I_t$ we select those with 
	$I_{t'_1},\ldots,I_{t'_k} \in\{i,i'\}$ and define a super-martingale with differences $Y_j = \II{I_{t'_j}=i'}-\II{I_{t'_j}=i}$ for $t'_j < \nbi$ and $Y_j=0$ for $t'_j \geq \nbi$. Then 
	\begin{align*}
		&\pr{T_i(s,t) < \cv T_{i'}(s,t) \wedge  t < \nbi \wedge T_i(s,t) \geq \Ci \log (n/\delta)} 
		\\&= \pr{\frac{1}{4}[T_i(s,t)+T_{i'}(s,t)] < \frac{3}{4} [T_{i'}(s,t)-T_i(s,t)] \wedge  t < \nbi \wedge T_i(s,t) \geq \Ci \log (n/\delta)} 
		\\&\leq \sum_{k \geq 2\Ci\log (n/\delta)} \pr{\frac{1}{4}k < \frac{3}{4} [T_{i'}(s,t)-T_i(s,t)] \wedge  t < \nbi \wedge T_i(s,t) + T_{i'}(s,t) = k} 
		\\&\leq \sum_{k \geq 2\Ci\log (n/\delta)} \pr{\frac{k}{3} < \sum_{j=1}^k Y_j} 
		\\&\leq \sum_{k \geq 2\Ci\log (n/\delta)} \exp\left\{   
		-\frac{k}{18} \right\} 
		\leq \exp\left\{-\frac{\Ci}{9} \log (n/\delta)\right\}
		\frac{1}{1-\exp\{-1/18\}}
		%
	\end{align*}
	by Hoeffding-Azuma's inequality (Lemma~\ref{l:max}).
	Since $\Ci \geq 45$, a union bound for $s$, $t$, $i$, and $i'$ completes the proof.
\end{proof}

\begin{proof}{\bf of (\ref{item:initT})}
	There is an $i'$ with $T_{i'}(s,t) \geq (t-s+1)/K$, and (\ref{item:diffTi}) implies~(\ref{item:initT}). 
\end{proof}
\fi

\begin{proof}{\bf of (\ref{item:numPhases})}
	We fix some arm~$i$. By the condition in Step~4.c, Step~4.b can be executed for this arm at most $E^0$ times. Let $m$ be the number of executions of Step~4.d for arm~$i$, such that the number of phases is at most $m+E^0+1$ and the length of the longest phase is at least $2^{m-E_0-1} \cdot L_i^0$. Then $n \geq \sum_{k=1}^{m+E^0} (\tau_{i,k+1}-\tau_{i,k}) \geq 2^{m-E^0}-1 + 2E^0$ and
	$m \leq E^0 + \lfloor \log_2 (n-1) \rfloor \leq E^0 + \lceil \log_2 n \rceil - 1$.
\end{proof}

\begin{proof}{\bf of (\ref{item:TiB})}
	We fix some arm~$i$ and use Bernstein's inequality (Lemma~\ref{l:bernstein}) with the martingale differences 
	\[ Y_t=\II{I_t=i} - p_i(t) \] 
	for $\nbi \leq t \leq \ns$ and $Y_t=0$ otherwise.
	Then $Y_j \leq 1$ and 
	\[ \sum_{t=1}^n \bE[Y_t^2|\hist_{t-1}] = \sum_{t=1}^n \bE[Y_t^2|p_i(t)] 
	\leq \sum_{t=\nbi}^{\ns} p_i(t) . \] 
	In any testing phase $k$, $p_i(t)=L_i^0/(K L_{ik})$ for $\tau_{i,k} \leq t < \tau_{i,k+1} \leq \tau_{i,k}+L_{ik}$. 
	Thus in each phase $\sum_{t=\tau_{i,k}}^{\tau_{i,k+1}-1} p_i(t) \leq L_i^0/K$
	and $\sum_{t=\nbi}^{\ns} p_i(t) \leq L_i^0 \numPH /K$.
	Hence Bernstein's inequality gives 
	\begin{align*} 
		&{\pr{T_i(\nbi,\ns) \geq (1+C)L_i^0\numPH/K} 
		 \leq  \pr{\sum_{t=1}^n Y_t \geq CL_i^0\numPH/K}} 
		\\& \leq  \exp\left\{-\frac{C^2 L_i^0\numPH/K}{2+2C/3}   \right\} 
		\leq  \exp\left\{-\frac{2 C^2 \Cpp \ci}{2+2C/3}\log(n/\delta)   \right\}
		\leq \delta/n 
	\end{align*}
	for $C \geq 1/100$.
	A union bound for $i$ completes the proof.
\end{proof}

\section{Proofs for \sapo\ against adversarial bandits}

\subsection{Proof of inequality (\ref{eq:bound-active})}
\label{app:bound-active}

We need to show that 
\begin{align*}
	\EE{\sum_{t=1}^{\nbi-1} \left[ x_i(t)-\lcb^*(t) \right]} 
	= \bigO{\sqrt{Kn\log (n/\delta)}} .
\end{align*}
By the definition of $\bmu_i(t)$ and by Step 1.a of \sapo\  we have by Wald's equation that
\[
\EE{\sum_{t=1}^{\nbi-1} x_i(t)} = \EE{(\nbi-1) \cdot \bmu_i(\nbi-1)} \leq \EE{(\nbi-1) \cdot \bucb_i(\nbi-1)}  .
\]
Since $\blcb_i(t) \leq \lcb^*(t)$ and $\bucb_i(t)$ is non-increasing, 
\begin{align*} 
	& \EE{\sum_{t=1}^{\nbi-1} \left[ x_i(t)-\lcb^*(t) \right]} 
	\leq \EE{\sum_{t=1}^{\nbi-1} \left[ \bucb_i(t)-\blcb_i(t) \right]} 
	\\ & \leq 2\EE{\sum_{t=1}^{\nbi-1} \bwidth(t)} 
	= 2  \EE{\sum_{t=1}^{\nbi-1} \sqrt{\frac{2\Cwid K\log (n/\delta)}{t}}}
	\leq 4 \sqrt{2\Cwid Kn\log (n/\delta)} .
\end{align*}

\subsection{Proof of inequality (\ref{eq:bound-bad})}
\label{app:bound-bad}

\begin{lemma} \label{l:qadv}
	We fix some phase $k$ and $s \geq \tau_{i,k}$. Let 
	\begin{align}
		\tC(s)=\min\{s \leq t < \tau_{i,k+1}: D_i(s,t) > \CVV \egp_i L_{ik}\} . \label{eq:tau}
	\end{align}
	If no such $t$ exists, we set $\tC(s)=\tau_{i,k+1}-1$. Then 
	\[ \pr{\left. D_i(s,\tC(s)) > \CVV \egp_i L_{ik} 
		\wedge \hD_i(s,\tC(s)) < \Caa \egp_i L_{ik} p_{ik} 
		\right| \hist_{s-1}} 
	\leq \qadv := 1/25 .\]
\end{lemma}
\begin{proof}
	We use Bernstein's inequality for martingales (Lemma~\ref{l:bernstein}) 
	on the martingale differences 
	\[ Y_t=p_{ik}[x_i(t)-\emu_i] - \II{I_{t}=i} [x_i(t)-\emu_i] \]
	for $s \leq t \leq \tC(s)$ and $Y_j=0$ otherwise; with $b=1$ and $V=p_{ik}L_{ik} = L_i^0/K$. We get
	\begin{align*} 
		&\pr{\left. D_i(s,\tC(s)) > \CVV \egp_i L_{ik} 
			\wedge \hD_i(s,\tC(s)) < \Caa \egp_i L_{ik} p_{ik} \right| \hist_{s-1}} 
		\\&\leq \pr{\left.p_{ik}D_i(s,\tC(s)) - \hD_i(s,\tC(s)) > \Caa \egp_i L_{ik} p_{ik}\right| \hist_{s-1}}
		\\&= \pr{\left.p_{ik}D_i(s,\tC(s)) - \hD_i(s,\tC(s)) > \Caa\egp_i L_i^0/K\right| \hist_{s-1}}\\
		&\leq \exp\spa{-\min\{\Cpp\Caa^2/4,\Cpp\Caa/2\}} \leq 1/25 .
	\end{align*}
\end{proof}

\begin{lemma} \label{l:qk}
	Consider some phase $k$. Then 
	\[ \PPik{D_i(\tau_{i,k},\tau_{i,k+1}-1) \geq m (\CVV \egp_i L_{ik} + 1)} \leq \qad{m}  .\]  	
\end{lemma}
\begin{proof}
	Since $D_i(s,t+1) - D_i(s,t) \leq 1$, 
	$D_i(\tau_{i,k},\tau_{i,k+1}-1) \geq m (\CVV \egp_i L_{ik} + 1)$ 
	implies that there are time steps 
	$\tau_{i,k}=s_1 < s_2 < \cdots < s_{m+1} \leq \tau_{i,k+1}$ 
	with 
	$D(s_{j},s_{j+1}-2) \leq \CVV \egp_i L_{ik}$ 
	and $\CVV \egp_i L_{ik} < D(s_{j},s_{j+1}-1) \leq \CVV \egp_i L_{ik}+1$.
	Furthermore, by the condition in Step~4.a, $\hD_i(s_j,t) < \Caa \egp_i L_{ik} p_{ik}$ for $j=1,\ldots,m$ and $s_j \leq t < \tau_{i,k+1}$ (otherwise the phase would have ended before $\tau_{i,k+1}$). We define the event 
	\begin{eqnarray*} 
		\ND_j = \{s_{j+1} = \tC(s_j) + 1  
		& \wedge & D_i(s_{j},s_{j+1}-1) > \CVV \egp_i L_{ik} 
		\\&\wedge& \hD_i(s_{j},s_{j+1}-1) < \Caa \egp_i L_{ik} p_{ik} \}.
	\end{eqnarray*}
	Then 
	\begin{align*} 
		\PPik{D_i(\tau_{i,k},\tau_{i,k+1}-1) \geq m (\CVV \egp_i L_{ik} + 1)}
		&\leq \PPik{\bigwedge_{j=1}^m \ND_j} 
		\\&= \prod_{j=1}^m \PPik{
			\ND_j 	\left| \bigwedge_{j'=1}^{j-1} \ND_{j'} \right. } 
		\\&\leq \qad{m}
	\end{align*}	
	by Lemma~\ref{l:qadv}.	
\end{proof}

\begin{lemma} \label{l:expD}
	For any phase $k$,
	\[
	\EEik{D_i(\tau_{i,k},\tau_{i,k+1}-1)} 
	\leq \frac{\CVV \egp_i L_{ik} + 1}{1-\qadv} 
	\]
\end{lemma}
\begin{proof}
	\begin{align*}
		&\EEik{D_i(\tau_{i,k},\tau_{i,k+1}-1)} 
		\\&\leq (\CVV \egp_i L_{ik} + 1)  \sum_{m \geq 0} 
		\PPik{D_i(\tau_{i,k},\tau_{i,k+1}-1) \geq m (\CVV \egp_i L_{ik} + 1)} 
		\\&\leq \frac{\CVV \egp_i L_{ik} + 1}{1-\qadv} 
	\end{align*}
	by Lemma~\ref{l:qk}. 
\end{proof}

\section{Proofs for \sapo\ against stochastic bandits}

\subsection{Proof of Lemma~\ref{l:sto-prelim}}
\label{app:sto-prelim}

We show that (\ref{item:wid-bar}) and (\ref{item:wid}) hold with probability $1-\bigO{\delta}$. The other statements of the lemma follow from the events in (\ref{item:wid-bar}) and (\ref{item:wid}).

\smallskip
\begin{proof}{\bf of (\ref{item:wid-bar}) and (\ref{item:wid})}
	We fix some step $t$ and some arm $i$, and condition on $T_i(t)=T$. Using  Hoeffding-Azuma's inequality~(Lemma~\ref{l:max}) we find
	\begin{align*}
		\pr{\hmu_i(t) - \mu_i > \width_i(t)/2 | T_i(t)=T}
		&\leq \exp\left\{ -\Cwid \log(n/\delta)/2 \right\}
		\leq \delta/(16 K n^2) .
	\end{align*}
	Analogously we bound $\mu_i - \hmu_i(t)$. A union bound over $t$, $i$, and $T$ gives (\ref{item:wid}).
	
	Since $i \in \cA(t)$ implies $p_i(t) \geq 1/K$, Bernstein's inequality (Lemma~\ref{l:bernstein}) with $b=K$ and $V=Kt$ gives 
	\begin{align*}
		\pr{\bmu_i(t) - \mu_i > \bwidth(t)/2}
		&\leq \exp\left\{ -\frac{\Cwid  \log(n/\delta)}{4 (2+2/3)} \right\}
		\leq \delta/(16 K n) .
	\end{align*}
	Using the same bound for $\mu_i - \bmu_i(t)$ and summing over $t$ and $i$ gives (\ref{item:wid-bar}).
\end{proof}

\subsection{Proof of Lemma~\ref{l:Delta>}}
\label{app:Delta>}

\newcommand{\iprim}{{i}}
\newcommand{\ijust}{{i'}}

\begin{lemma} \label{l:diffTi}
	With probability $1-\bigO{\delta}$ the following holds for all time steps $t$ and all arms $\iprim$, $\ijust$:
	If $\ijust \in \cA(t)$ and $T_{\iprim}(t) \geq \Ci \log (n/\delta)$, then $T_{\ijust}(t) \geq T_{\iprim}(t)/4$.
\end{lemma}	
\begin{proof} 
	We fix $t$, $\iprim$, and $\ijust$. By the construction of \sapo\  we have 
	$\pr{I_t=\ijust|\hist_{t-1},\ijust \in \cA(t)} \geq \pr{I_t=\iprim|\hist_{t-1},\ijust \in \cA(t)}$.
	From $I_{s},\ldots,I_t$ we select those with 
	$I_{t'_1},\ldots,I_{t'_k} \in\{\iprim,\ijust\}$ and define a super-martingale with differences $Y_j = \II{I_{t'_j}=\iprim}-\II{I_{t'_j}=\ijust}$ for $t'_j \leq t$ and $Y_j=0$ for $t'_j > t$. Then 
	\begin{align*}
		&\pr{T_\ijust(t) < T_{\iprim}(t)/4 \wedge  \ijust \in \cA(t) \wedge T_\iprim(s,t) \geq \Ci \log (n/\delta)} 
		\\&= \pr{\frac{3}{8}[T_\ijust(s,t)+T_{\iprim}(s,t)] < \frac{5}{8} [T_{\iprim}(s,t)-T_\ijust(s,t)] 
			\wedge  \ijust \in \cA(t) \wedge T_\iprim(s,t) \geq \Ci \log (n/\delta)} 
		\\&\leq \sum_{k \geq \Ci\log (n/\delta)} \pr{\frac{3}{8}k < \frac{5}{8} [T_{\iprim}(s,t)-T_\ijust(s,t)] \wedge \ijust \in \cA(t) \wedge T_\ijust(s,t) + T_{\iprim}(s,t) = k} 
		\\&\leq \sum_{k \geq \Ci\log (n/\delta)} \pr{\frac{3k}{5} < \sum_{j=1}^k Y_j} 
		\\&\leq \sum_{k \geq \Ci\log (n/\delta)} \exp\left\{   
		-\frac{9k}{50} \right\} 
		\leq \exp\left\{-\frac{9\Ci}{50} \log (n/\delta)\right\}
		\frac{1}{1-\exp\{-9/50\}}
	\end{align*}
	by Hoeffding-Azuma's inequality (Lemma~\ref{l:max}).
	A union bound for $t$, $\iprim$, and $\ijust$  completes the proof.
\end{proof}

\begin{proof}{\bf of Lemma~\ref{l:Delta>}}\\
	Let arm $i^*$ be optimal, $\mu_{i^*}=\mu^*$, such that $i^* \in \cA(t)$ by Lemma~\refSP{lcb*}.
	By Lemma~\refSP{wid}, with probability $1-\bigO{\delta}$ we have $|\hmu_i(t-1) - \mu_i| \leq \width_i(t-1)/2$ for arms~$i$ and~$i^*$.
	By construction, $\lcb^*(t-1) \geq \lcb_{i^*}(t-1) \geq \hmu_{i^*}(t-1) - \width_{i^*}(t-1)$.
	By Lemma~\ref{l:diffTi}, with probability $1-\bigO{\delta}$ we have $T_{i^*}(t-1) \geq T_{i}(t-1)/4$.
	Then
	\begin{align*}
		\Delta_i &= \mu^* - \mu_i 
		\\&\leq \hmu_{i^*}(t-1) + \width_{i^*}(t-1)/2 - \hmu_i(t-1) + \width_i(t-1)/2
		\\&\leq \lcb^*(t-1) + 3\width_{i^*}(t-1)/2 - \hmu_i(t-1) + \width_i(t-1)/2
		\\&\leq (\Cgap + 3 + 1/2) \width_i(t-1) 
		\\&\leq 2\Cgap\cdot\width_i(t-1) .
	\end{align*}
\end{proof}

\subsection{Considering Step 1.b}

\begin{lemma}\label{l:step1b}
	The probability that there is a time $t$ with $\sum_{s=1}^{t-1} [\lcb^*(s) - x_{I_s}(s)] > \Cib\sqrt{Kn\log (n/\delta)}$ is at most $\bigO{\delta}$.
\end{lemma}
\begin{proof}
	By Lemma~\refSP{lcb*} we have $\lcb^*(s) \leq \mu^*$ for all $s$ with probability  $1-\bigO{\delta}$. Thus (\ref{eq:Tns}) implies that with probability $1-\bigO{\delta}$, 
	\begin{align*}
		&\sum_{s=1}^{t-1} (\lcb^*(s) - \EE{x_{I_s}(s)|\hist_{s-1}}) 
		\leq \sum_{s=1}^{t-1} (\mu^*(s) - \EE{x_{I_s}(s)|\hist_{s-1}}) 
		\\&= \sum_{i=1}^K \Delta_i T_i(t-1) 
		\leq \sum_{i=1}^K C \frac{\log(n/\delta)}{\Delta_i} 
		\leq \frac{C}{\Cib}\sqrt{K n \log (n/\delta)} 
	\end{align*}
	for $C > \frac{101}{100} \Cpp (2\ci+1) + 4 \Cwid \Cgap^2$.
	By Hoeffding-Azuma's inequality (Lemma~\ref{l:max}) we also have 
	\begin{align*}
		\pr{\max_{1 \leq t \leq n} \sum_{s=1}^{t-1} (x_{I_s}(s) - \EE{x_{I_s}(s)|\hist_{s-1}}) 
			\geq \sqrt{2 n \log(n/\delta)} }
		\leq (\delta/n) \leq 3\delta/4 .
	\end{align*}
	Thus the lemma follows for $\Cib \geq 522$ which satisfies $\Cib \geq C/\Cib+1$.
\end{proof}

\subsection{Considering Step 4.a}

\begin{lemma}\label{l:qs} If the statements in Lemma~\ref{l:sto-prelim} hold, then
	\[ 
	\PPik{\mbox{The condition of Step 4.a is triggered for arm $i$ in its phase $k$}} \leq \qsto := 0.21.
	\] 
	%With increasing $\Cpp$, $\qsto$ becomes arbitrarily small.  
\end{lemma}
\begin{proof}
	The probability of triggering the condition in phase $k$ is  
	\[
	\PPik{\max_{\tau_{i,k} \leq s \leq t < \tau_{i,k+1}} \hD(s,t) \geq \Caa \egp_i L_{ik} p_{ik}} . 
	\]
	We first bound the number of plays in this round, $T_i(\tau_{i,k},\tau_{i,k+1}-1)$. Applying Bernstein's inequality (Lemma~\ref{l:bernstein}) with $b=1$, $V=L_{ik}p_{ik}$, and 
	$z=L_{ik}p_{ik}$ we get
	\[
	\PPik{T_i(\tau_{i,k},\tau_{i,k+1}-1) \geq 2L_{ik}p_{ik}}
	\leq \exp\left(-\frac{L_{ik}^2 p_{ik}^2}{2L_{ik} p_{ik} + 2L_{ik} p_{ik}/3}\right)
	\leq \exp\left(-\frac{3\Cpp}{8\egp_i^2} \right) .
	\]
	By Lemma~\refSP{wid} and Step~2.b, $\mu_i - \emu_i \leq \width_i(\nbi-1)/2 \leq \egp_i/(2\Cgap)$.
	Conditioning on $T_i(\tau_{i,k},\tau_{i,k+1}-1) < 2L_{ik}p_{ik}$ and applying Corollary~\ref{cor:maxmax} of Hoeffding-Azuma's inequality with  
	\[ z=\Caa \egp_i L_{ik} p_{ik} -(\mu_i - \emu_i)2L_{ik} p_{ik}=(\Caa  -1/\Cgap)\egp_iL_{ik} p_{ik} \]
	yields
	\begin{align*}
		&\PPik{\left. \max_{\tau_{i,k} \leq s \leq t < \tau_{i,k+1}} \hD(s,t) \geq \Caa \egp_i L_{ik} p_{ik}
			\right| T_i(\tau_{i,k},\tau_{i,k+1}-1) < 2L_{ik}p_{ik}} 
		\\&\leq 2\exp\left(-\left(\Caa-\frac{1}{\Cgap}\right)^2\frac{(\egp_i L_{ik} p_{ik} )^2}{4L_{ik}p_{ik}}\right)
		\leq 2\exp\left(-\left(\Caa-\frac{1}{\Cgap}\right)^2\frac{\Cpp}{4}\right) 	
		\leq 0.21
	\end{align*}
\end{proof}

\end{document}